\pdfoutput=1
\relax
\documentclass[letterpaper]{article} 
\usepackage{aaai22}  
\usepackage{times}  
\usepackage{helvet}  
\usepackage{courier}  
\usepackage[hyphens]{url}  
\usepackage{graphicx} 
\usepackage{natbib}  
\usepackage{caption} 
\DeclareCaptionStyle{ruled}{labelfont=normalfont,labelsep=colon,strut=off} 
\usepackage{mathtools}

\newcommand{\OKGD} {OKGD\xspace}

\usepackage{graphics}
\usepackage{trimclip} 

\usepackage[linesnumbered,ruled,vlined]{algorithm2e}					  
\usepackage{enumitem}
\usepackage{float}

\usepackage{amsthm}
\newtheorem{theorem}{Theorem}

\newtheorem{corollary}{Corollary}
\newtheorem{assumption}{Assumption}
\newtheorem{assumption_2}{Assumption A.}
\newtheorem{theorem_2}{Theorem A.}

\usepackage{verbatim}
\usepackage{multirow}

\usepackage{xr}
\usepackage{balance}

\usepackage{url}
\usepackage{amsmath}
\usepackage{amssymb}

\usepackage{mathtools}

\usepackage{dsfont}
\usepackage{upgreek}
\usepackage{bbm}			

\usepackage{graphicx}
\usepackage{epstopdf}
\graphicspath{ {./figures/} } 
\usepackage{tikz}
\usepackage{float}

\usepackage{booktabs} 

\usepackage[labelfont=bf]{caption}
\usepackage{subfigure}

\usepackage{xspace}

\usepackage[scaled=.7]{beramono}

\DeclareRobustCommand\sampleline[1]{%
  \tikz\draw[#1] (0,0) (0,\the\dimexpr\fontdimen22\textfont2\relax)
  -- (1.4em,\the\dimexpr\fontdimen22\textfont2\relax);%
}

\makeatletter
\g@addto@macro\bfseries{\boldmath}
\makeatother

\newcommand{\prechange}{{\text{pre}}}
\newcommand{\postchange}{{\text{post}}}
\newcommand{\Sec}[1]		{Sec.\,\ref{#1}}
\newcommand{\Fig}[1]		{Fig.\,\ref{#1}}

\newcommand{\Eq}[1]			{Eq.\,\ref{#1}}
\newcommand{\Expr}[1]			{Expr.\,\ref{#1}}

\newcommand{\Tab}[1]		{Tab.\,\ref{#1}}
\newcommand{\Alg}[1]		{Alg.\,\ref{#1}}
\newcommand{\Theorem}[1]{Theorem\,\ref{#1}}

\newcommand{\Assumption}[1]{Assumption\,\ref{#1}}

\newcommand{\Problem}[1]{Problem\,\ref{#1}}	
\newcommand{\ie}   			{i.e.\@\xspace}
\newcommand{\eg}   			{e.g.\@\xspace}
\newcommand{\etc}   		{etc.\xspace}

\newcommand{\one}       {\mathbf{1}}     
\newcommand{\ones}[1]   {\one_{#1}}
\newcommand{\zero}      {\mathbf{0}}
\newcommand{\zeros}[1]   {\zero_{#1}}

\newcommand{\diag}    {\operatorname{diag}}
\newcommand{\real}      {\mathbb{R}}

\newcommand{\pdf}		{p.d.f.\xspace}
\newcommand{\pdfs}		{p.d.f.'s\xspace}

\newcommand{\Hilbert}      {\mathbb{H}}

\newcommand{\inlinetitle}[2]  {\vspace{4pt}\noindent\textbf{\emph{#1}{#2}}}

\renewcommand*{\top}{{\mkern-1.5mu\mathsf{T}}}

\newcommand{\Expec}[1]{\mathbb{E}[#1]}
\newcommand{\ExpecN}[1]{\mathbb{E}_{p(y)}[#1]}

\newcommand{\Expecn}{\mathbb{E}_{p(y)}}
\newcommand{\ExpecNm}[2]{\mathbb{E}_{p_{#2}(y)}[#1]}
\newcommand{\ExpecA}[1]{\mathbb{E}_{p'(y)}[#1]}
\newcommand{\ExpecNA}[1]{\mathbb{E}_{(p(y),p'(y))}[#1]}

\newcommand{\ExpecAm}[2]{\mathbb{E}_{p'_{#2}(y)}[#1]}
\newcommand{\norm}[1]{\left\lVert#1\right\rVert}

\newcommand{\dott}[2]{\langle #1,#2 \rangle}
\newcommand{\tra}[1]{#1^{\top}}

\newcommand{\Lpc}{\mathlarger{\mathcal{L}}}
\newcommand{\id}{\mathbb{I}}

\newcommand{\dtheta}{\Delta_{\theta_v}}
\newcommand{\abs}[1]{\left|#1\right|}
\newcommand{\Npre}{n^{\prechange}}
\newcommand{\Npost}{n^{\postchange}}
\newcommand{\setpre}{Y_t^{\prechange}}
\newcommand{\setpost}{Y_t^{\postchange}}
\newcommand{\setprebef}{Y_{t-1}^{\prechange}}
\newcommand{\bp}{q} 
\newcommand{\nghood}{\text{ng}} 
\newcommand{\Gdims}{N} 

\newcommand{\explanation}[1]{\hfill\footnotesize(\text{#1})}
\makeatletter
\newcommand{\pushright}[1]{\ifmeasuring@#1\else\hfill$\displaystyle#1$\fi\ignorespaces}
\newcommand{\pushleft}[1]{\ifmeasuring@#1\else$\displaystyle#1$\hfill\fi\ignorespaces}
\makeatother

\DeclareMathOperator*{\R}{\mathbb{R}}

\newcommand{\NOTE}[1] 				{{\color{red} #1}}

\newcounter{marginNoteCounter}

\title{Online non-parametric change-point detection\\ for 
heterogeneous data streams observed over graph nodes}
\author{
Alejandro de la Concha \qquad Argyris Kalogeratos \qquad Nicolas Vayatis 
}
\affiliations{

    Center Borelli, ENS Paris-Saclay, Université Paris-Saclay, 91190 Gif-sur-Yvette, France \\
    alejandro.de\_la\_concha\_duarte@ens-paris-saclay.fr
    nicolas.vayatis@ens-paris-saclay.fr
    argyris.kalogeratos@ens-paris-saclay.fr
}

\begin{document}

\maketitle

\begin{abstract}

Consider a heterogeneous data stream being generated by the nodes of a graph. The data stream is in essence composed by multiple streams, possibly of different nature that depends on each node. %
At a given moment $\tau$, a change-point occurs for a subset of nodes $C$, signifying the change in the probability distribution of their associated streams. In this paper we propose an online non-parametric method to infer $\tau$ based on the direct estimation of the likelihood-ratio between the post-change and the pre-change distribution associated with the data stream of each node. We propose a kernel-based method, under the hypothesis that connected nodes of the graph are expected to have similar likelihood-ratio estimates when there is no change-point. We demonstrate the quality of our method on synthetic experiments and real-world applications. 
\end{abstract}

\section{Introduction}
\label{sec:intro}

Change-points delimit segments in which a data stream exhibits a notion of stationarity. 
Their detection is fundamental for time-series analysis and control tasks. 
Modern challenges include handling larger amounts of more complex data streams, a clear example of which is when those lie over a graph. %
Many real-world systems can be seen as a network in which each node generates a stream of data. %
For instance, activity or monitor data can be collected from users in a social network, stations in a railway, or banks in a financial system. %
%
%
 A change-point in these systems may signify the shift of users' interest
, a disruption of the railway service, or the early signs of an economic crisis. Using any available information to detect as fast as possible such changes is important. 
%
The aim of this paper is to address the naturally arising question: how can we integrate the latent information carried by the graph structure in a change-point detection task? 

\inlinetitle{Related work}{.} Change-point detection methods can be classified into \emph{offline} \cite{truong2020,Aminikhanghahi2016} and \emph{online} approaches \cite{Tartakovsky2014,Tartakovsky2021} depending on whether we have access to a complete dataset or we observe data in near real-time. Further classification can be done into parametric, semi- and non-parametric approaches, depending on the assumptions which are made regarding the distribution generating the data stream. The Online Kernel Graph Detector (\OKGD) presented in this paper is a online non-parametric approach. 

Kernel methods have been used 
in both the offline \cite{Harchaoui2008,Garreau2018,Arlot2019}
and the online settings \cite{Kawahara2012,Bouchikhi2019,Li2019}, showing useful theoretical and practical properties. Kernel-based methods are appealing because they can be applied to complex data 
and monitor different kinds of changes, such as mean shifts, or changes in the correlation between variables. 
More related to our work is the kernel-based algorithm presented in \cite{Kawahara2012} that estimates directly the likelihood-ratio between the post-change-point and the pre-change-point probability density function of a data stream. A change is spotted as soon the estimator is bigger than a given threshold. Its theoretical properties 
under a semi-parametric hypothesis are analyzed in \cite{Kanamori2012}. 

The existing literature has hardly studied change-point detection for \emph{heterogeneous data streams}, where each node may generate a stream of different dimensions and/or nature (\eg some nodes may generate sensor signals, others text or other content, \etc). %
To the best of our knowledge, \cite{Ferrari2020} is the only paper referring to this setting. Its authors make the hypothesis that the underlying graph has a community structure and a change may occur in only one of these communities at a given moment. Using tools originating from Graph Signal Processing \cite{Ortega2018,Perraudin2017}, they spot 
the change-point and the affected community. %
%
Their approach consists in estimating independently the likelihood-ratio for each of the nodes. This essentially composes a graph signal, which is then filtered with the Graph Fourier Scan Statistic (GFSS) \cite{Sharpnack2016} in order to detect the community of interest. The main disadvantage of this approach is its limited applicability to cases where the graph has poor community structure (\eg graphs approximating a manifold structure). Furthermore, the sophistication of GFSS is limited, as it is a graph filter whose only effect is to filter out the high frequencies while maintaining the low frequencies. 
Therefore, scenarios where changes may concern nodes that are less concentrated at a site of the graph (or even arbitrary to it) is a blind spot to this approach. 


\inlinetitle{Main contributions}{.} %
In this paper we present the Online Kernel Graph Detector (\OKGD) algorithm that builds upon the notion \emph{graph smoothness}, which formalizes the intuition that two nodes are expected to have a similar behavior if they are connected. 
%
%
%
We introduce a kernel-based cost function made of two terms:
a penalized LSE-like term aiming to infer the likelihood-ratio at each of the nodes, and a Laplacian penalization term aiming to guarantee the smoothness of the estimations. Then, we develop a stochastic gradient descent strategy to minimize the cost function. %

In a nutshell, 
the \OKGD approach: i) is online,  ii) non-parametric, meaning it is flexible and can handle heterogeneous data streams, ii) exploits the expected smoothness of the graph signal derived by the likelihood-ratio estimations, and iv) is easy to parallelize and its computational time improves 
with the sparsity of the graph. 

\section{Notation and definitions} \label{sec:notation}

Let $x_i$ be the $i$-th entry of vector $x$, and $A_{ij}$ be the entry associated with the $i$-th row and $j$-th column of a matrix. $\ones{\Gdims}$ represents the vector with $\Gdims$ ones (resp. $\zeros{\Gdims}$), and $\id_{\Gdims}$ is the $\Gdims\times \Gdims$ identity matrix (the subscript may be omitted). Let $G=(V,E,W)$ be a weighted and undirected graph without self-loops, where $V$ is the set of vertices, $E$ the set of edges, and $W \in \real^{\Gdims\times \Gdims}$ 
its 
adjacency matrix. 

Let us suppose we observe $\Gdims$ synchronous data streams that could be of different nature, each of them being generated by a node of $G$. We denote by $y_{v,t}$ the observation generated by node $v$ at time $t$, by $y_t=(y_{1,t},...,y_{\Gdims,t})$ the vector obtained after concatenating all the observations from all the nodes at time $t$.  We will denote by $Y=\{y_{t}\}_{t\in \{1,...\}}$ the whole time-series.

Additionally, we suppose that the graph $G$ (hence $W$) remains constant through time. We also suppose that $W$ is weighted and symmetric, $W_{uv}=W_{vu}$, and each $W_{uv}$ entry is a positive value reflecting how similar are expected to be the data streams generated by the nodes $u$ and $v$ (note: $W_{uu}=0, \forall u \in V$). %
%
The degree of $v$ is $d_v=\sum_{u \in \nghood(v)} W_{uv}$, where $u \in \nghood(v)$ means that $u$ is a neighbor of $v$.

A term that we use throughout the paper is that of \emph{graph signal} which is attributed to any function $x:V \rightarrow \R$ defined from the nodes of a graph into $\R$. Consider as change-point the timestamp at which the distribution associated with a time-series changes meaning its probability density function (\pdf) goes from $p(y)$ to $p'(y)$. 
%
Here, a change-point is denoted by $\tau$, which indicates a change in the marginal distribution of a subset of nodes $C$, more explicitly:
\begin{equation}
    \left\{
    \begin{array}{ll}
        t < {\tau} \ \  y_{v,t} \sim p_v(y);  \\
        t \geq {\tau} \ \  y_{v,t} \sim p_v'(y),
    \end{array}
    \right.
\end{equation}
where $p_v(y) \neq p'_v(y)$ if ${v \in C}$, otherwise $p_v(y) = p'_v(y)$. We consider $p_v(y),p'_v(y)$ as well as $C$ and $\tau$ o be unknown. We denote by $p(y)$ (resp. $p'(y)$) the joint \pdf of $Y$ before the change-point (resp. after the change-point), meaning its $v$-marginal \pdf is $p_v(y)$ (resp. $p'_v(y)$). A simple example of this framework is illustrated in \Fig{fig:CP_GS}.

Similarly, we denote by $\ExpecN{\cdot}$ (resp. $\ExpecA{\cdot}$) the expected value under the pre-change (resp. post-change) distribution. For a set of observations $\{y_t\}_{t \in I}$ in a given time interval $I$, 
we denote by $\ExpecNA{\cdot}$ the expectation on the joint \pdf of the time-series at that interval. If there is a change
, then $p(y) \neq p'(y)$, otherwise $p(y) = p'(y)$.

In this work, we estimate and monitor the likelihood-ratio between $p'_v(y)$ and $p_v(y)$ at each node, that is the quantity $r_v(t)= \frac{p'_v(y)}{p_v(y)}$. We define by $r(y)=(r_1(y),...,r_{\Gdims}(y))$ the vector whose entries are the likelihood-ratios corresponding to each node of $G$. Note that $r(y)$ is a graph signal. 
In the text, we call \emph{null hypothesis} 
the case where there is no change, \ie when $p_v(y)=p'_v(y)$, $\forall v \in V$, which opposes the \emph{alternative hypothesis} where a change exists.

 
The \emph{smoothness} of a graph signal $x$ over a graph $G$ can be defined using a Laplacian operator $\Lpc$ associated with $G$. %
%
%
Here, we use the combinatorial Laplacian operator, $\Lpc=\diag(d_v)_{v \in V}-W$, for which smoothness is defined as: 
\begin{equation}{\label{eq:smooth2}}
s(x) = x^\top \Lpc x %
= \frac{1}{2}  \sum_{u,v \in V} W_{uv} (x_u-x_v)^2.
\end{equation}%
%
$s(x)$ is always positive, and is lower when adjacent values are more similar.

\begin{figure}[t!]
  \centering
  \includegraphics[width=0.7\linewidth, viewport=0 250 900 720, clip]{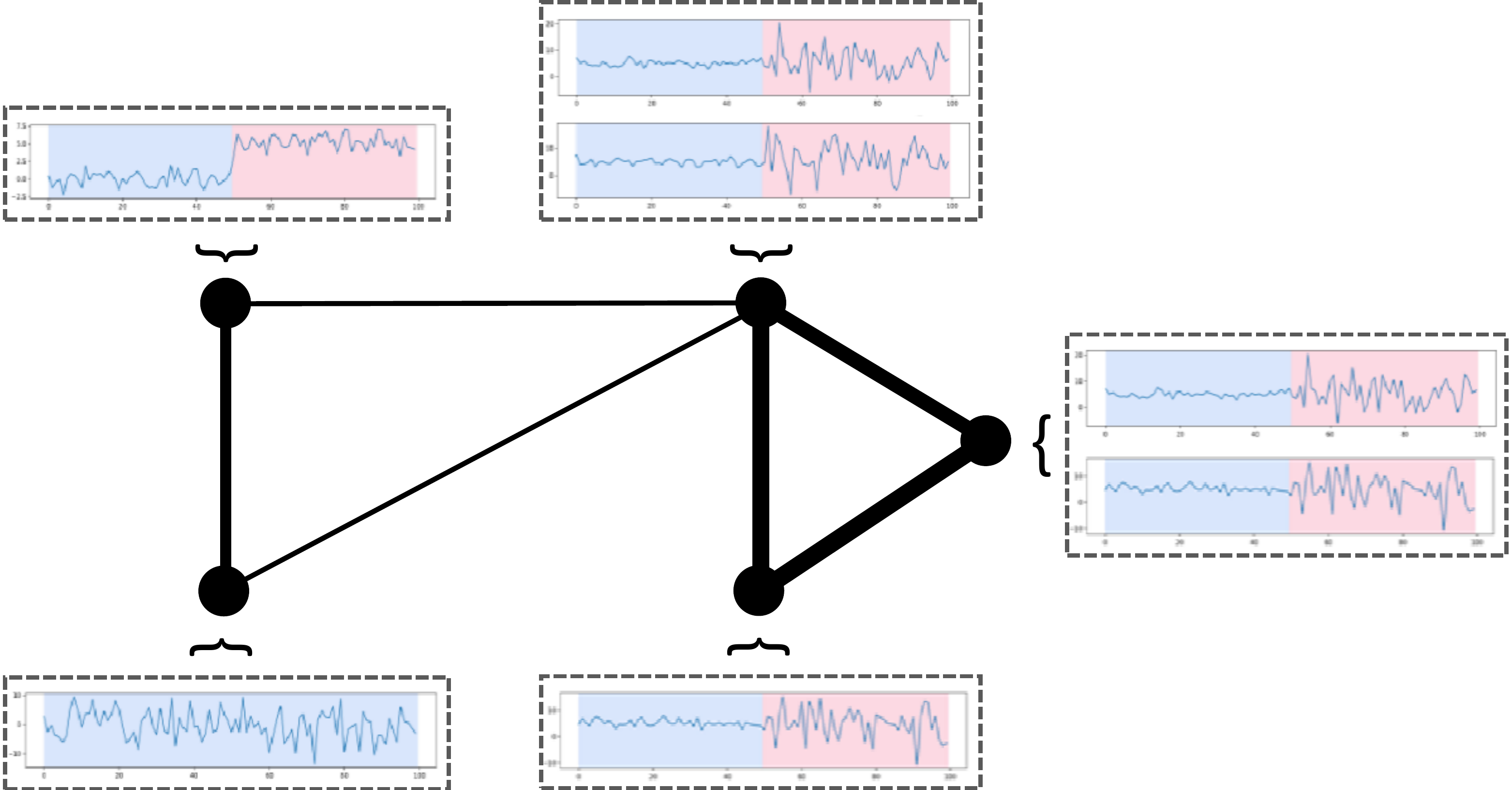}\hspace{4mm}%
\vspace{-0.5mm}
    \caption{\small Example of a heterogeneous data stream, observed over the nodes of a weighted graph that expresses node behavior similarity. The dimensions or even the nature of the time-series of each node may vary. A change occurs in a subset of nodes and the change-point is at the moment when the color changes 
in the respective time-series. 
For the top-left node, the change-point is associated with a shift in the mean, while the three nodes of the clique on the right present a shift in the variance of their streams.}
\label{fig:CP_GS}
\end{figure}

\section{Problem formulation and solution}\label{sec:Problem_formulation}

Our approach relies on a cost function that would estimate the likelihood-ratio $r_v$ associated with each node, and would integrate all such estimates. We desire the solution of the cost-function to have a similar behavior to that of $r_v$ under the null hypothesis, so that we can detect when new observations come from a different probabilistic model. Furthermore, we want to be able to solve the optimization problem in a way that allows the online update of the parameters.

Estimating the likelihood-ratio $r_v$ at each node $v$ requires a dictionary and a kernel function associated with $v$. In our approach, the estimation of the $r_v$'s of all nodes is performed jointly, however we suppose that a dictionary and a kernel function can be defined and learned independently for each node. The quality of the dictionary affects the performance of the algorithm. This is why we propose to integrate the approach of \cite{Cedric2009} in order to enlarge our dictionary with arriving observations that are dissimilar to what has been previously observed.

\subsection{Cost function}\label{subsec:estimation}
%
Most usually we do not know the form of the joint \pdfs, before the change-point $p(y)$ and after it $p'(y)$; same for its marginals $p_v(y)$ and $p'_v(y)$ for each node $v$. Instead, our method uses a kernel-based approach to estimate the graph signal $g(y):r(y)-\ones{\Gdims}=(r_1(y_1)-1,...,r(y_{\Gdims})-1)^\top$. %
Under the null hypothesis of no change-point, it holds $g(y) = \zeros{\Gdims}$, hence $s(g(y)) = 0$, see \Eq{eq:smooth2}). For this reason, in our cost function we include a penalization term against the non-smoothness of $g(y)$: 
%
%
\begin{equation}\label{eq:cost_function}
\begin{aligned}
\Phi(g) & =   \Expecn \left[ \frac{\norm{r(y)-\one_{\Gdims}-g(y)}_2^2}{2}+\frac{\lambda}{2} g(y)^\top \Lpc  g(y) \right] \\
     & =  \frac{1}{2} \ExpecN{\norm{g(y)}_2^2}+\ExpecN{\tra{\one_{\Gdims}}g(y)}-\ExpecA{\tra{\one_{\Gdims}}g(y)} \\
    & + \frac{\lambda}{4} \Expecn \left[ \sum_{u,v \in V} 
		W_{uv} \left(   g_u(y) -  g_v(y) \right)^2  \right]  + K,
\end{aligned}
\end{equation}
where $K$ is a constant. Notice that the initial expression of \Eq{eq:cost_function} works under the distribution $p(y)$ before the change-point (null hypothesis) and consists of a least square term that controls the approximation error, and a penalization term aiming to promote the smoothness of $g(y)$. Then, the second equality is due to 
%
    $\ExpecN{r_v(y)g_v(y)}= \ExpecAm{g_v(y)}{v}$,
which is is easy to verify given the definition of $r(y)$.

The next step is to approximate $\Phi(g)$ 
by means of empirical expectations. 
It is common that a user has a sample of $\Npre$ observations generated by the probabilistic model $p(y)$, then $\Npost$ new observations arrive and she has to decide if those last observations are still coming from $p(y)$. We denote by $Y_t$ the set of all observations coming from $p(y)$ and have occurred up to the decision time $t$, and $\setpost$ the last observations that need to be validated. 
 Let us suppose that the user takes a sub-sample $Y_t^{\prechange} \subseteq Y_t$ of size $\Npre$, and observes $\Npost$ recent observations up to time $t$: 
%
%
%
\begin{equation}\label{eq:windows}
\left\{
    \begin{array}{ll}
        Y_t^{\prechange}\!\! &= \{y_{1},...,y_{\Npre}\}, \ \forall y_1,...,y_{\Npre} \sim p(y);  \\
        Y_t^{\postchange\!\!\!\!} &= \{y_{t-\Npost},...,y_{t-1},y_{t}\}.
    \end{array}
    \right.
\end{equation}
%
We can now define our optimization problem 
in terms of its empirical expected values: 
\begin{equation}\label{eq:empirical_problem}
\hspace{-2.5mm}
\begin{aligned}
& \min_{(g_1,...,g_{\Gdims}) \in H_1 \times H_2,...,\times H_{\Gdims}}  \sum_{v \in V}  \left[ \frac{\sum_{y_j \in\setpre} g^2_v(y_{v,j})}{2 \Npre} \right.  \\
& \left. + \left( \frac{\sum_{y_j \in\setpre} g_v(y_{v,j})}{\Npre} - \frac{\sum_{y_j \in\setpost} g_v(y_{v,j})}{\Npost} \right) \right] \\ & + \frac{\lambda}{2}\left[ \frac{  \sum_{y_j \in\setpre} \sum_{u,v \in V} W_{uv}  \left(  g_v(y_{v,j}) -  g_v(y_{u,j}) \ \right)^2  }{\Npre} \right] 
\\ & + \sum_{v \in V} \frac{\gamma}{2} \norm{g_v}^2_{H_v},\!\!\!\!\!\!\!\!\!\!
\end{aligned}
\end{equation}
where $H$ denotes the product of $N$ reproducing kernel spaces $H:H_1 \times H_2 \times ... \times H_{\Gdims}$. Each $H_v$ represents the reproducing kernel space related with the data stream of node $v$ with kernel function $k_v(x_v,y_v): H_v \times H_v \rightarrow \R$. The term $\norm{g}^2_{H_v}$ is an extra penalization term which favors sparse representations of $g_v(y_v)$ in terms of a basis of $H_v$. By applying the Representer Theorem in this context we can show that every solution of \Problem{eq:empirical_problem} satisfies:
\begin{equation}\label{eq:kernelsum}
  g_v(\cdot,\theta_v) = \sum_{j={t-(\Npost+\Npre-1)}}^{t} \theta_{v,k} k_v(\cdot,y_{v,j}).
\end{equation}
In this way, the problem will be solved once we find the parameters $\theta_{v,j}, \forall v \in V \text{ and } \forall j \in \{t-(\Npost+\Npre-1), ..., t\}$. Having a dictionary that is rich enough to represent each of the spaces $H_v$, $\forall v$, we can reduce the complexity of \Eq{eq:kernelsum} to: 
\begin{equation}\label{eq:kerneldict}
 g_v(\cdot,\theta_v) = \sum_{l={1}}^{L_v} \theta_{v,l} k_v(\cdot,y_{v,l}),
\end{equation}
where $L_v$ is the size of the dictionary describing the data stream generated by the node $v$, $\theta_v=(\theta_{v,1},...,\theta_{v,L_v})^\top$ and $\theta=(\theta_1^\top,...,\theta_{\Gdims}^\top)^\top \in \R^{L}$, where $L=\sum_{v \in V} L_v$. We let $k_v(y_{v,j})=(k_v(y_{v,j},y_{v,0}),...,k_v(y_{v,j},y_{v,L_v}))^\top$ and we define the matrix $K_{G}(y_j) \in \R^{L \times \Gdims}$ as:
\begin{equation}
K_{G}(y_j) = %
\begin{pmatrix}
k_1(y_{1,k}) & 0 & ... & 0 \\\
0 &  k_2(y_{2,k}) & ... & ...  \\\
... &  ... & ... & ... \\\
0 &  0 & ... & k_2(y_{\Gdims,k})
\end{pmatrix}\!\!.\!\!
\end{equation}

\Eq{eq:kerneldict} allows us to rewrite the optimization \Problem{eq:empirical_problem} in terms of the parameter $\theta$. Let us introduce the partial terms: 
\begin{equation}
\begin{aligned}\label{eq:updatehs}
 h_{v,t}^{\postchange} & = \frac{\sum_{y_j \in\setpost} k_v(y_{v,j})}{\Npost}, 
h_{v,t}^{\prechange}  = \frac{\sum_{y_j \in\setpre} k_v(y_{v,j})}{\Npre} \in \real^{L_v}, \\
H_{v,t}^{\prechange}  & =\frac{\sum_{y_j \in\setpre} k_v(y_{v,j})k_v(y_{v,j})^\top}{\Npre} \in  \real^{L_v \times L_v}, \\
h^{\postchange}_t  & = (h_{1,t}^{\postchange}, h_{2,t}^{\postchange},... h_{\Gdims,t}^{\postchange}),
h^{\prechange}_t  = (
  h_{1,t}^{\prechange},  h_{2,t}^{\prechange},..., h_{\Gdims,t}^{\prechange}) \in \real^{L}, \\
H^{\prechange}_t  & = \begin{pmatrix}
H^{\prechange}_{1,t} & 0 & ... & 0 \\\
0 &  H^{\prechange}_{2,t} & ... & ...  \\\
... &  ... & ... & ... \\\
0 &  0 & ... & H^{\prechange}_{\Gdims,t}  
 \end{pmatrix} \in \real^{L \times L} .
\end{aligned}
\end{equation}
%
 %
Finally, we rewrite \Problem{eq:empirical_problem} to the following minimization of a function $F_t(\theta)$ that has an elegant quadratic form (see details in the Appendix):
\begin{equation}\label{eq:final_problem}
    \min_{\theta \in \R^L } F_t(\theta)  = \min_{\theta \in \R^{L} } \frac{ \theta^\top  A_t \theta}{2} + \theta^\top b_t,
\end{equation}
%
%
\begin{align}
    \!\!\!\!\!\!\!\!\text{where}\ \  A_t &= H_t^{\prechange}+  \lambda  \sum_{y_j \in\setpre } \frac{ K_G(y_j) \Lpc K_G(y_j)^\top}{\Npre} + \gamma \id_{L},\!\! \\
    b_t &= h_{t}^{\prechange} - h_{t}^{\postchange}.
\end{align}
Since the matrix $A_t$ is positive definite, \Problem{eq:final_problem} has a unique solution for each $t$.

\subsection{Online estimation}

At each time $t$, we generate $\setpre$ from the set $Y^{\prechange}$ and use the last $\Npost$ observations to compute $F_t$. With these elements we can solve the optimization \Problem{eq:final_problem}. Nevertheless, in many applications the size of the graph is large, making the inversion of the matrix $A_t$ prohibitive. Furthermore, we know that under the null hypothesis, all $F_t$, $\forall t$, have the same expectation. (Later, we will prove that the expected value is $0$). We can then define the following cost function:
\begin{equation}\label{eq:expected_cost}
\begin{aligned}
 \min_{\theta} \textbf{F}(\theta)  &=  \min_{\theta} \ExpecNA{F_t(\theta)}\\
 &= \min_{\theta \in \R^{L} } \frac{ \theta^\top \ExpecNA{ A_t} \theta}{2}  +  \theta^\top \ExpecNA{b_t},
\end{aligned}
\end{equation}
\Eq{eq:expected_cost} suggests that $F_t$ can be seen as a noisy observation of the cost function $\textbf{F}(\theta)$. We would like to minimize $\textbf{F}(\theta)$ using the information given by each cost function $F_t$ when $p(y)=p'(y)$. The plan is to follow a stochastic gradient descent-like strategy that we know it will converge to $0$, and spot a change-point when the update exceeds a given threshold. We apply the Block Stochastic Gradient Method (BSGD) \cite{Xu2015} that combines the advantage of stochastic gradient descent with block coordinate descent to solve problems with multiple blocks of variables. Also, BSGD achieves an optimal order of convergence rate.

In our problem, each block of variables is associated with each node $v$, thus it contains $\theta_v$. Before applying BSGD, we require to define a fixed update order for the variables. Let $\theta_{<v}$ be the set of variables that were updated before $v$, and $\theta_{ \geq v}$ be the complement of that set. BSGD requires a noisy observation of the partial derivative of the cost function $\textbf{F}$. In our case, such noisy gradient will be the derivative of the cost function $F_t$ with respect to $\theta_v$. Then, the estimate at time $t+1$ ($\theta_{v,t+1}$) can be updated according to: 
\begin{equation}\label{eq:update}
\begin{aligned}
\theta_{v,t+1} & =\theta_{v,t}-\alpha_{v,t} \big( \Delta_{\theta_{v}} F_t((\theta_{t+1})_{<v},(\theta_{t})_{ \geq v}) \big) \\
 & = \theta_{v,t}-\alpha_{v,t} \big( B_{v,t} \theta_{v,t} + c_{v,t} \big)\\
 & = \theta_{v,t}-\alpha_{v,t} \Bigg[ \Big( (1+\lambda d_v) H_{v,t}^{\prechange}+ \id_{L_v} \Big) \theta_{v,t} \\
 & -  \frac{\lambda}{\Npre} \sum_{y_j \in\setpre}\sum_{u \in \nghood(v)} W_{uv}   \left[ k_v(y_{v,j})k_u(y_{v,j})^\top \right.
\\ &\left. \big( \theta_{u,t+1} \one_{u<v} + \theta_{u,t} \one_{u \geq v} \big)\right]+  h_{v,t}^{\prechange}-h_{v,t}^{\postchange} \Bigg],
\end{aligned}
\end{equation}%
where $\Delta_{\theta_v,t} F_t(\cdot)$ denotes the partial derivative of the cost function $F_t$, and $\alpha_{v,t}$ is the learning rate associated with node $v$ at time $t$. It is worth pointing out that the update of $\theta_{v}$ only require information from the neighbors of $v$. 

In order to reduce this expression, we introduce the terms:
\begin{equation}\label{eq:udate}
\begin{aligned}
    & B_{v,t}  =(1+\lambda d_v) H_{v,t}^{\prechange}+ \id_{L_v} \\
    & c_{v,t} =  h_{v,t}^{\prechange}-h_{v,t}^{\postchange}  -  \frac{\lambda}{\Npre} \sum_{y_j \in\setpre}\sum_{u \in \nghood(v)} W_{vu}   \left[ k_v(y_{v,j})k_u(y_{u,j})^\top \right.
\\ &\left. \big( \theta_{u,t+1} \one_{u<v} + \theta_{u,t} \one_{u \geq v} \big)\right].
\end{aligned}
\end{equation}
The properties of this estimator are analyzed in the Theoretical Analysis section 
In particular, we prove the convergence of the update of \Eq{eq:update} under the null hypothesis. 

Once the parameters $\theta_{v,t+1}$ has been estimated we compute the statistic $\hat{g}_t:(\hat{g}_{1,t},...,\hat{g}_{t,n})$:
\begin{equation}\label{eq:score}
    \hat{g}_t= \frac{1}{\Npre} \sum_{y_j \in\setpre}\theta_{t+1} K_G(y_j) \in \real^{\Gdims}.
\end{equation}
$\hat{g}_t+\ones{\Gdims}$ is an approximation of the likelihood-ratio at time $t$. Furthermore, this is an asymptotic unbiased estimator when 
$t \rightarrow \infty$. We finally detect a change-point when $\norm{\hat{g}_t}>\epsilon$. 

\subsection{Dictionary learning}

In the previous section, the cost function $F_t(\theta)$ required having $N$ dictionaries, each one describing the Hilbert space $H_v$ associated with node $v$. This can be hard in practice, especially since we expect that at some point the observations will start being different to what will have been seen till then. Therefore, it is important to enrich the dictionary with new observations in an online manner, which would improve the quality of $\hat{g}_t$. %
This specific problem has been studied in the context of online prediction for time-series \cite{Cedric2009}. In that paper, the authors address the problem of including new observations to a dictionary while controlling its size and the sparse representation of the datapoints. They achieve that by introducing the \emph{coherence} measure that quantifies the similarity of the new point to the elements already present in the dictionary. If this is smaller that a given threshold $\mu_0$, the new datapoint is added into the dictionary, hence the dimension of the dictionary $L_v$ will increase its size by one and, as a consequence, $\theta_v$'s dimension increases as well. The update of \Eq{eq:update} reflects that, as also shown in \Alg{alg:OKGD}.

\begin{algorithm}[!t]
\footnotesize
\SetAlgoLined
\SetKwInOut{Input}{input}
\SetKwInOut{Output}{output}
\Input{
$\lambda$, $\gamma$: penalization constants \\
$\mu_0$: coherence parameter \\
$\bp$: how many of the first observations to be used to \\ \ \ \ \ generate the initial dictionary of each node\\
$\Npre$, $\Npost$: sizes of the detection windows
\\
$\epsilon$: threshold used to detect a change-point
}
\Output{$\hat{\tau}$: the detection time of the change}

\raisebox{0.25em}{{\scriptsize$_\blacksquare$}}~\textbf{Dictionary initialization}

\For{$v \in \{1,...,\Gdims\}$} {
$L_v = 1$

Observe $k_v(y_{v,t_1}) := k_v(y_{v,1})$ and add it to the dictionary\!\!\!\!\!\!\!\!\!\!\!\!

\For{$t \in \{2,...,\bp\}$}{
Observe $y_{v,t}$

\For{$v \in \{1,...,\Gdims\}$} {
	Update $B_{v,t}$ and $c_{v,t}$ (\Expr{eq:update})
}

\If{$\max_{l \in \{1,...,L_v \}} k_v(y_{v,t},y_{v,l}) \leq \mu_0$}{
$L_v := L_v + 1$

Add $y_{v,L_v} := k_v(y_{v,t})$ to the dictionary
}
}
Create the set $Y^{\prechange} = \{y_{j}\}_{j=1}^{\bp}$
}
\raisebox{0.25em}{{\scriptsize$_\blacksquare$}}~\textbf{Online estimation and detection}

\For{$t \in \{ \bp+N^{\mathrm{\postchange}},...\}$} {
Sample $\Npre$ observations from $Y^{\prechange}$ and update the set 
$\setpre$\!\!\!\!\!\!\!\!\!\!\!\!

Observe $y_{v,t}$ and update the sliding window $Y_t^{\postchange}$ (\Eq{eq:windows})\!\!\!\!\!\!\!\!\!\!

\raisebox{0.25em}{{\scriptsize$_\square$}}~\textbf{Dictionary update}

\uIf{$\max_{l \in \{1,...,L_v \}} k_v(y_{v,t},y_{v,l}) \leq \mu_0$}{

$L_v = L_v + 1$

Add $k_v(y_{v,t_{L_v}}) := k_v(y_{v,t})$ to the dictionary

$\vartheta = \begin{bmatrix} \theta_{v,t-1}\\ 0 \end{bmatrix}$

}

\Else{

$\vartheta = \theta_{v,t-1}$

}

\raisebox{0.25em}{{\scriptsize$_\square$}}~\textbf{Parameters update}

\For{$v \in \{1,...,\Gdims\}$} {

$C_{v,t} = (1+\lambda d_v)  \norm{H_{v,t}}_2+\gamma
+ \lambda M_v \sum_{u \in \nghood(v)} W_{uv}  M_u$\!\!\!\!\!\!\!\!\!\!\!\!\!

$\alpha_{v,t} =\min\big(\frac{c}{t-(\bp+\Npost-1)},\frac{1}{C_{v,t}}\big)$

$\theta_{v,t} = \vartheta - \alpha_{v,t}\left( B_{v,t} \vartheta + c_{v,t} \right)$\!\!\!\!\!\!\!\!
}

%
%
\raisebox{0.25em}{{\scriptsize$_\square$}}~\textbf{Online detection}

Compute the score $\hat{g}_t:=(\hat{g}_{1,t},...,\hat{g}_{\Gdims,t})$

\uIf{ $\norm{\hat{g}_t} > \epsilon$}{
A change-point is detected at $\hat{\tau}=t$
}
\Else{
$Y^{\prechange} = Y^{\prechange} \cup \{y_{t-\Npost}\}$
}
}
\caption{Online Kernel Graph Detector ($\OKGD$)}{\label{alg:OKGD}}
 \vspace{-1mm}
\end{algorithm}

\section{Theoretical analysis}\label{sec:Theorical_Analysis}

In this section, we show how the cost function $\textbf{F}$ of \Eq{eq:expected_cost} admits a unique solution $\theta^{*}$, and also that the induced estimator $g(\cdot,\theta^{*})$ is an unbiased estimator of the vector of likelihood-ratios $r(y)-\ones{\Gdims}$ under the null hypothesis.

As we have pointed out earlier, we do not have access to all the data observations, so it is impossible to estimate $\theta^{*}$ directly. For this reason we analyze the convergence of the estimator $\theta_t$ 
(\Eq{eq:update}) under the null hypothesis, as well as the asymptotic behavior of the associated estimator $\hat{g}_t$ 
(\Eq{eq:score}).

Here, we focus only on the case where we have a fixed dictionary, which means no elements are added later and the dimension of $\theta$ remains constant through time.

\begin{assumption}\label{assum:kernel}

The Hilbert space $H_v$ is separable and there exists a constant $M_v>0$ such that:
\begin{equation}\label{eq:em}
\sup_{(x,y) \in H_v x H_v} k_v(x,y) \leq M_v < \infty.
\end{equation}
\end{assumption}
\noindent\Assumption{assum:kernel} is true for Gaussian and Laplacian kernels, and has 
been used in \cite{Harchaoui2008}.

The fact that the space $H_v$ is separable and the boundness assumption over $k_v$ allow us to define the Bochner mean $\mu_{v} \in \Hilbert_v$ also called the mean element $p_v(y_{v})$ which is defined as the element such that $\forall h_v \in \Hilbert_v$: 
\begin{equation}
\begin{aligned}
\dott{\mu_{v}}{y_{v}}_{\Hilbert_v} &=\ExpecNm{h_v(y_{v})}{v}  = \ExpecNm{\dott{y_{v}}{h_v}_{\Hilbert_v}}{v}.\!\!\!\!
\end{aligned}
\end{equation}

\vspace{-1mm}
\begin{assumption}\label{assum:independence}
The observations of the time-series $y_1,..,y_t$ are independent in time.
\end{assumption}

This is a standard hypothesis in kernel-based change-point detection literature \cite{Arlot2019,Li2019,Harchaoui2008,Bouchikhi2019}.

\begin{theorem}\label{Th:main_result}
\Problem{eq:expected_cost} admits a unique solution $\theta^*$. Furthermore, under the null hypothesis, it holds $\theta^*=0$.
\end{theorem}

The proof of \Theorem{Th:main_result} (see Appendix) does not require \Assumption{assum:kernel} nor \Assumption{assum:independence}. The main implication is that the estimator defined in \Eq{eq:kerneldict} $g_v(,\theta_v^*)=0$ $\forall v \in V$. This means that, when no change has occurred, $g_v(\cdot,\theta_v^*)$ is an unbiased estimator for $r_v(\cdot)-1$ . 

\begin{theorem}\label{Th:convergence}
Consider the case when $p(y)=p'(y)$ and \Assumption{assum:kernel} and \Assumption{assum:independence}. Let $\theta_{t+1}$ be generated as described in \Eq{eq:update} with $\alpha_{v,t}<\frac{1}{L}$ $\forall v,t$. If $\sigma=\sup_t \sigma_t < \infty $ and there exists $M$ as defined in \Eq{eq:em}, then: 
\begin{equation}
    \ExpecN{\norm{\theta_{t+1}}^2} \leq  \frac{\ExpecN{\norm{\theta_{t}}^2}}{1+\frac{1}{2} \gamma \alpha_t}+2D \frac{\alpha_t^2}{1+\frac{1}{2} \gamma\alpha_t}.
\end{equation}
\noindent In particular when $\alpha_{v,t}=\frac{c}{t} < \frac{1}{L}$ $\forall v \in V,t$:
\begin{equation}
    \ExpecN{\norm{\theta_{t+1}}^2} \leq  \frac{1}{t} \max\left\{\!\frac{4D c(1+\frac{1}{2} c \gamma)}{\gamma} ),\norm{\theta_{1}}^2\!\right\},\!\!\!
\end{equation}%
\begin{equation*}
\begin{aligned}
&\text{where} \ \ D=\frac{\Gdims \sigma^2}{1- L c} + \sqrt{\Gdims \rho^2}
\Big( A+ L \sqrt{ \Gdims(4 L^2 \rho^2  +4\sigma^2_t)} \Big), \\
& \lambda^{\text{\upshape{max}}}_v = \norm{\ExpecN{H_{v,t}}}_2 \text{  and  }
\end{aligned}
\end{equation*}
\begin{equation}
\begin{aligned}
\!\!\!\!\!\!L &= \max_{v \in V} \left\{\!(1\!+\!\lambda d_v) \lambda^{\text{\upshape{max}}}_v+\gamma  + \lambda M_v \!\!\!\sum_{u \in \nghood(v)}  W_{uv} M_u \!\right\}\!\!\!\!\!\\
&= \max_{v \in V} C_v. \\
\end{aligned}
\end{equation}
\end{theorem}

The proof of \Theorem{Th:convergence} is given in the Appendix. Notice that the first inequality gives an upper-bound for the expected value of the norm of the estimation $\theta_t$ at time $t$, which depends on: the previous estimation $\theta_{t-1}$, the step-size $\alpha_k$, and the constant $D$. $D$ 
depends on: the upper-bound of the kernel values, the approximation error $\sigma_t$ and the graph structure. Furthermore, we can see that by reducing the learning rate to linear, $\theta_t$ converges to zero, which is the value of $\theta^*$ under the null hypothesis. In practice, we achieve a good performance by following the recommendation of \cite{Xu2015} and by fixing the learning rate at $\alpha_{v,t}= \min\{\frac{c}{t}, \frac{1}{C_{v,t}} \}$, where $C_{v,t}$ is computed as in \Alg{alg:OKGD}.

Another important consequence of \Theorem{Th:convergence} is:

\begin{corollary}\label{cor:convergence}
Under the conditions described in \Theorem{Th:convergence}, the estimator $\hat{g}_t$ defined in \Expr{eq:score} is an unbiased asymptotic estimator of the vector $r(y)-\ones{\Gdims}$.
\end{corollary}

\section{Experiments}\label{sec:exps}

\subsection{Use-cases on synthetic data}

Our intention is to validate the use of the graph structure in the change-point detection task and to show the different scenarios in which our method is able to operate successfully. We design two synthetic scenarios caused by changes occurring with different spatial characteristics. %
Both scenarios use the same graph sampled from a Stochastic Block Model with $4$ clusters, C$1$,...,C$4$, each having $20$ nodes. The intra- and inter-cluster connection probability for node pairs is fixed at $0.5$ and $0.01$, respectively. %
To illustrate that we can handle data heterogeneity, the time-series of the nodes in each cluster are generated by different probabilistic models (see details in the Appendix). 

The change that we introduce in the time-series is as follows: the nodes of the cluster C$1$ start behaving as in C$3$, and vice-versa (\ie a change in the node correlation structure). Moreover, the nodes of C$2$ increase its expected value as in C$4$, while the nodes of C$4$ reduce it as in C$1$.

In the experiments we compare three algorithms. The first is the one of \cite{Ferrari2020} (NOUGAT-based): it consists in estimating the likelihood-ratio in a kernel-based way, and then applying a low-frequency Graph filter aiming to spot nodes having a likelihood-ratio different from zero, under the hypothesis that they will be elements of a cluster in the graph. 
It uses a constant learning rate for all nodes, and it monitors the quantity $\Vert G\!F\!S\!S(\hat{l}_t)\Vert$, where $G\!F\!S\!S(\cdot)$ is the Graph Fourier Scan Statistic and $\hat{l}_t =\theta_t K_G(y_{t+1})$. We also use: the \OKGD method with $\lambda=0$, which is when the graph structure is ignored (OKGD no graph), and finally \OKGD with $\lambda > 0$ 
that is our complete proposal integrating the graph in the detection algorithm. 

We fix the burning period at $\bp = 100$ observations that are used to generate the initial dictionary for each node. The coherence parameter is fixed at $\mu_0 = 0.5$. We use a Gaussian kernel with scale parameter tuned via the median heuristic over the $\bp$ observations. 
Moreover we set $\gamma = 10$, and $\Npre = \Npost = 100$ for the size of the windows. %
The detection threshold $\epsilon$ was set adeptly, as we know $\hat{g}_t$ will converge towards zero. We fix $\epsilon_t$ at $1.5$ times the mean of $\{\norm{\hat{g}}_j\}_{j=1}^{t}$. For the last algorithm (\OKGD), we fix $\lambda = \frac{10}{\bar{d}}$, where $\bar{d}$ is the average node degree in $G$. 

\inlinetitle{Changes at arbitrary random graph locations}{.} %
In the first experiment $10$ nodes at random suffer a change in their distribution, as described previously. 
An instance of this scenario can be seen in \Fig{fig:random_locations}. We can see that the NOUGAT-based approach fails as its ``changing cluster'' hypothesis is not satisfied. Contrary, our method is able to spot the change even when the graph is not taken into account. %
After averaging over $100$ instances of this scenario, we can see in \Tab{tab:synthetic experiments} how the \OKGD method performs the best in terms of the observed expected detection delay and precision. We can see, that the NOUGAT fails here to detect the change-point and has a precision of just $36 \%$.

%
\begin{figure}[t]
\centering
\vspace{-3mm}
\!\!\!\!\!\!\!\!\!\subfigure[Changes at random locations]{\label{fig:random_locations}
\includegraphics[width=0.245\textwidth, viewport=0 8 430 300, clip]{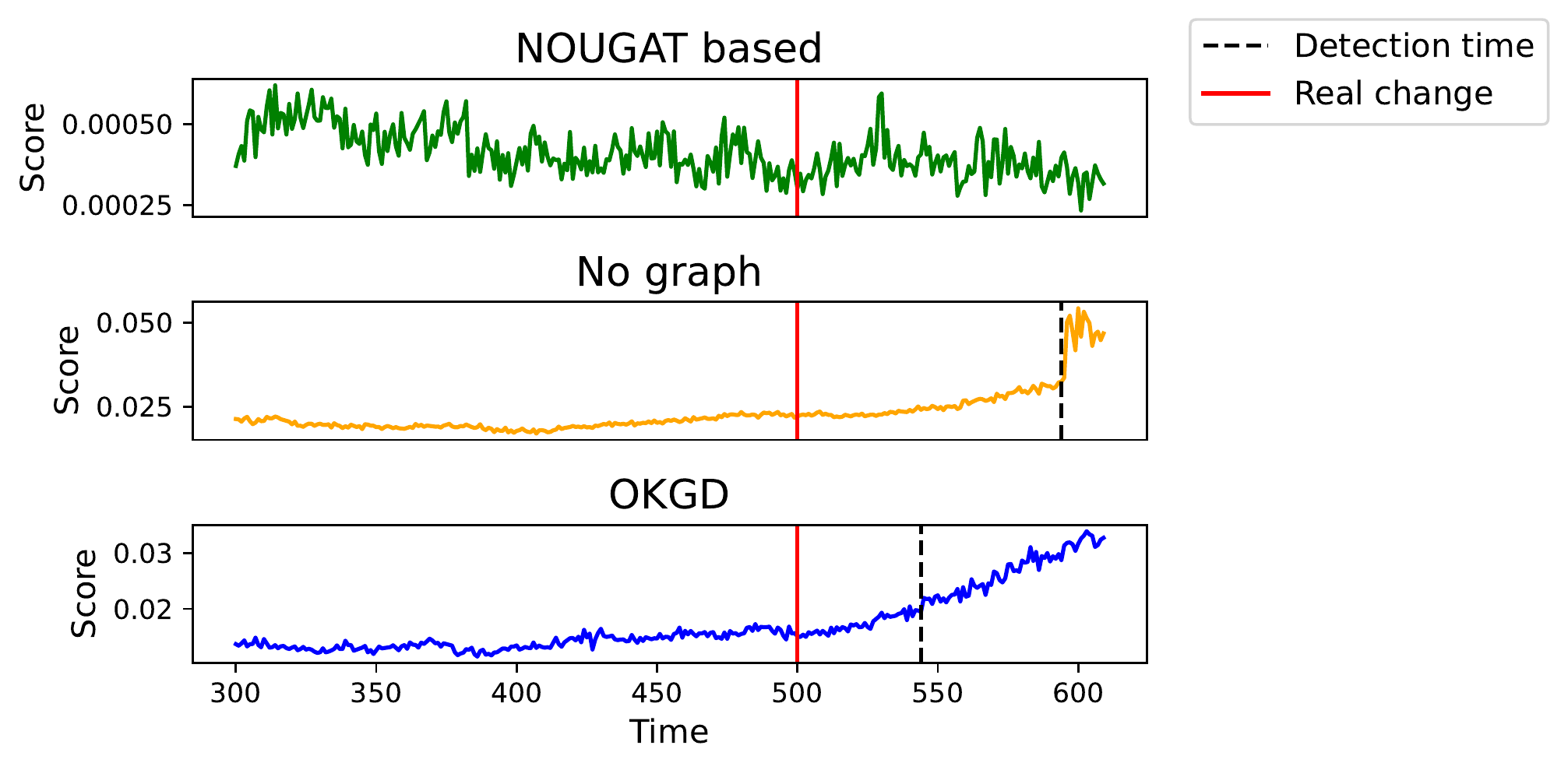}
}
\llap{\raisebox{2.5cm}{
      \includegraphics[width=0.07\textwidth, viewport=435 235 575 280, clip]{images/random.pdf}
    }}%
\subfigure[Changes in a graph cluster]{\label{fig:clusters}
\includegraphics[width=0.245\textwidth, viewport=0 8 430 300, clip]{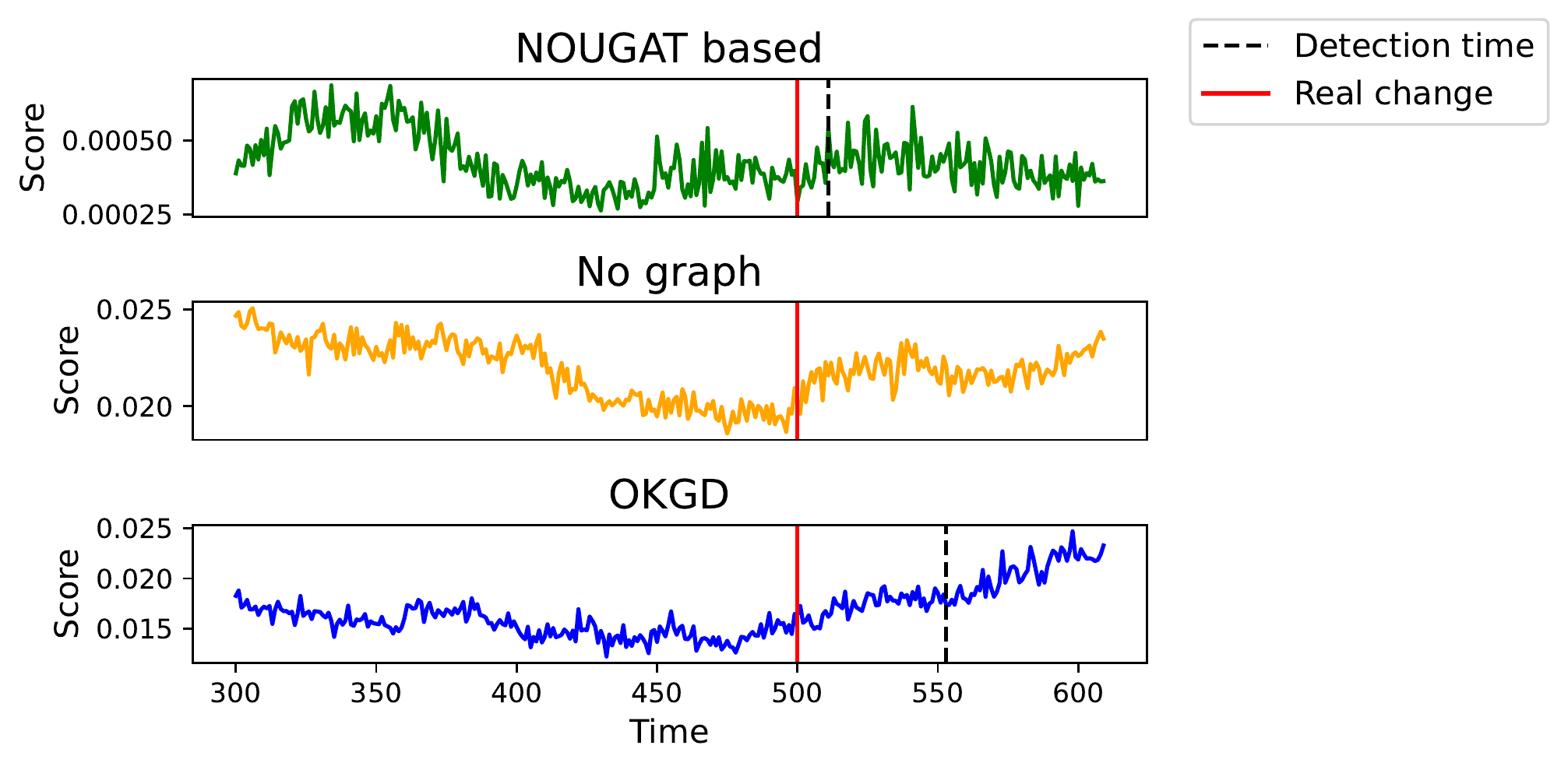}
}
\vspace{-2mm}
\caption{Comparison of the detection score of the algorithms: NOUGAT-based (green), \OKGD without graph structure (orange) and \OKGD (blue). Problem instances: a) the marginal density changes for $10$ nodes that are picked at random, %
b) the change affects the covariance matrices of all the $20$ nodes of the first cluster.}
\vspace{-2mm}
\label{fig:this_figure} 
\end{figure}


\inlinetitle{Changes concentrated in a graph cluster}{.} %
In the second experiment, all the $20$ nodes of one cluster change their distribution. This problem was also the motivation of the NOUGAT-based algorithm. What we investigate here is whether our algorithm is able to also identify this kind of changes. \Fig{fig:clusters} shows a simulated instance: the affected community is the C$1$ cluster, which at time $500$ starts behaving like the C$3$ cluster (\ie the variables become dependent). As we can see, the methods that use the graph structure are able to spot the change. In this particular simulation instance, the NOUGAT detector is the fastest.

\Tab{tab:synthetic experiments} shows results after averaging over $100$ simulation instances. In each instance, we first select a cluster at random and then we change the probabilistic model of all of its nodes. 
Interestingly, even in this particular case, \OKGD has an expected detection delay lower than the NOUGAT-based algorithm, a higher precision, and no false alarms.
 

\subsection{Use-cases on real data}

\inlinetitle{Seismic dataset}{.}
We show the performance of our method using the well-known dataset provided by the High Resolution Seismic Network. 
The dataset contains the tremor signal captured by the sensors located at $13$ seismic stations near Parkfield, California. Each of the stations contains three geophones in three mutually perpendicular directions. The recording was captured on Dec 23, 2004 from 2:00am to 2:16am with an observation every $0.064$ seconds, making a total of around $15k$ observations. An earthquake measured at duration magnitude $1.47$Md hit 
at 02:09:54.01. The goal of this exercise is to detect the moment the earthquake occurred. 
The dataset version we are working on is described in \cite{Chen2021} and is available online\footnote{The ocd package:
\url{https://cran.r-project.org/web/packages/ocd}.} . The processing of the data is detailed in the same paper; it comprises standard steps in the seismology literature and a autoregressive filter of order $1$.

We represent seismic stations as graph nodes. The graph is built via the $5$-nearest neighbor method. The signal generated at each node is the measurements of the three geophones, \ie $y_{v,t} \in \R^3$. Here we set $\Npre = \Npost = 100$, which corresponds to $6.4$ seconds. The threshold $e_t$ is fixed as $4$ times the empirical expectation of the score $\norm{\hat{g}_t}$.  The other parameters are tuned as described at the beginning of Sec.\,V. 

We can see how the graph structure reduces the detection delay in this case, as the results for \OKGD and NOUGAT-based show. Nevertheless, the number of false alarms is bigger for the NOUGAT-based. %
The epicenter of the earthquake occurred $50$km far away from the seismic stations and the wave travels at $6$km/s, so the detection delay of $12.07$ seconds is a sufficiently good result in this context. 

\begin{figure}[t]\label{fig:sysmic}
\vspace{-2mm}
\begin{minipage}[b]{\linewidth}
  \centering
  \includegraphics[width=\linewidth]{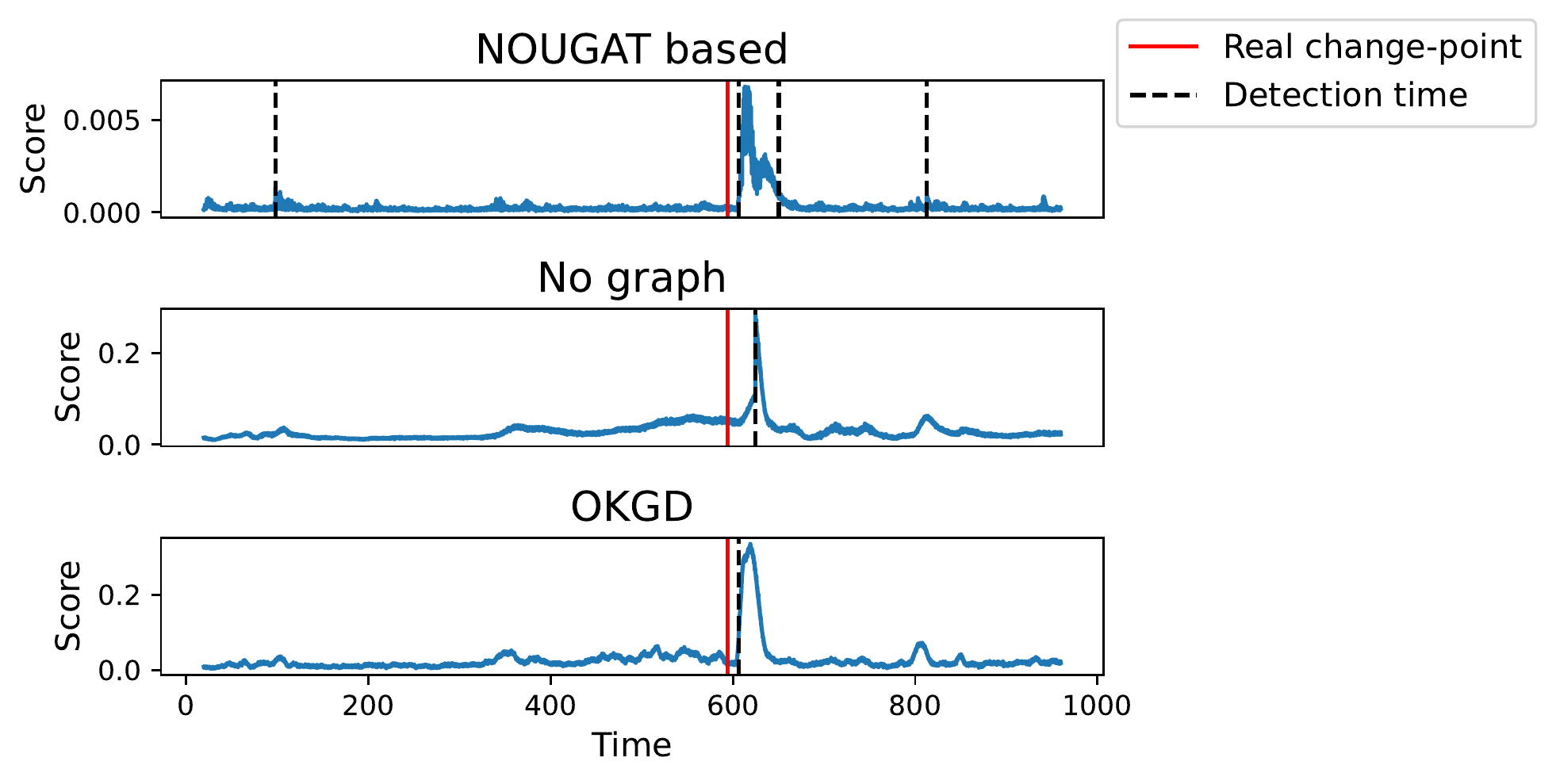}
\end{minipage}
\vspace{-7mm}
\caption{
Seismic dataset: the norm of the score generated by the different change-point detection methods at each time $t$.}
\vspace{-5mm}
\label{fig:}
\end{figure}

\inlinetitle{Influenza epidemics dataset}{.}
The French \emph{Sentinelles Network} \cite{Sentinelles2021} follows epidemic outbreaks and collects timely epidemiological data. 
%
It provides a weekly update of the number of cases and incidence rates of various re-occurring diseases in the different regions of France. Furthermore, it uses a simple yet standard methodology 
 to detect the starting point of an epidemic outbreak. 
The exercise is 
detect the beginning of such outbreaks.

This dataset 
includes the weekly incidence rates of influenza-like diseases per $100$k inhabitants. The 
time period goes from the $44$-th week of 1984 up to the $11$-th week of 2020 for the $21$ geographical regions of France. We build a graph containing nodes that represent the regions, and edges connecting regions that share border. The signal at each node is the recorded incidence rate of infection at one region. 

For the burning period, we use observations that have occurred before 1995, excluding the epidemic periods. We fix the threshold $e_t$ to be $1.75$ times the empirical expectation of the score $\norm{\hat{g}_t}$. We fix $\Npre=24$ that corresponds to a $6$-month period, and we try to test for outbreaks using upcoming observations for every two weeks $\Npost=2$. The other parameters are tuned as described at the beginning of Sec.\,V. 

The results are reported in \Fig{fig:flu} and \Tab{tab:synthetic experiments}. Once again, we can see that \OKGD shows better performance compared to the other methods: higher precision, smaller detection delay, and no false alarms. In this use-case, ignoring the graph leads to very big errors (see \Tab{tab:synthetic experiments}). On the other hand, the NOUGAT-based method has a low detection delay, a high precision, but it has the problem of more false alarms. It is very interesting to note how the \OKGD method anticipates the peak of the national incidence rates in most outbreaks. In this case, we allow a window of $15$ weeks around the official beginning of each epidemic period, as in some cases the official date occurs after the peak of the epidemic period.

\section{Conclusions and further work}

In this paper, we presented the Online Kernel Graph Detector for spotting change-points in data streams observed over graphs
. Our sound formalization of the problem uses a kernel-based cost function made of two terms: a penalized LSE-like term aiming to infer the likelihood-ratio at each of the nodes, and a Laplacian penalization term aiming to guarantee the smoothness of the estimations. This is a non-parametric approach flexible enough to handle heterogeneous data streams. We demonstrated the efficiency of the method in several synthetic and real scenarios. One of the hard points that remains to be analyzed is the choice of the  threshold parameter, which would require a broader theoretical framework able to evaluate also the significance of the graph structure in a specific detection problem. 

\begin{figure}[t]
\begin{minipage}[b]{\linewidth}
  \centering
  \centerline{
	\includegraphics[width=0.46\linewidth, viewport=0 9 565 440,clip]{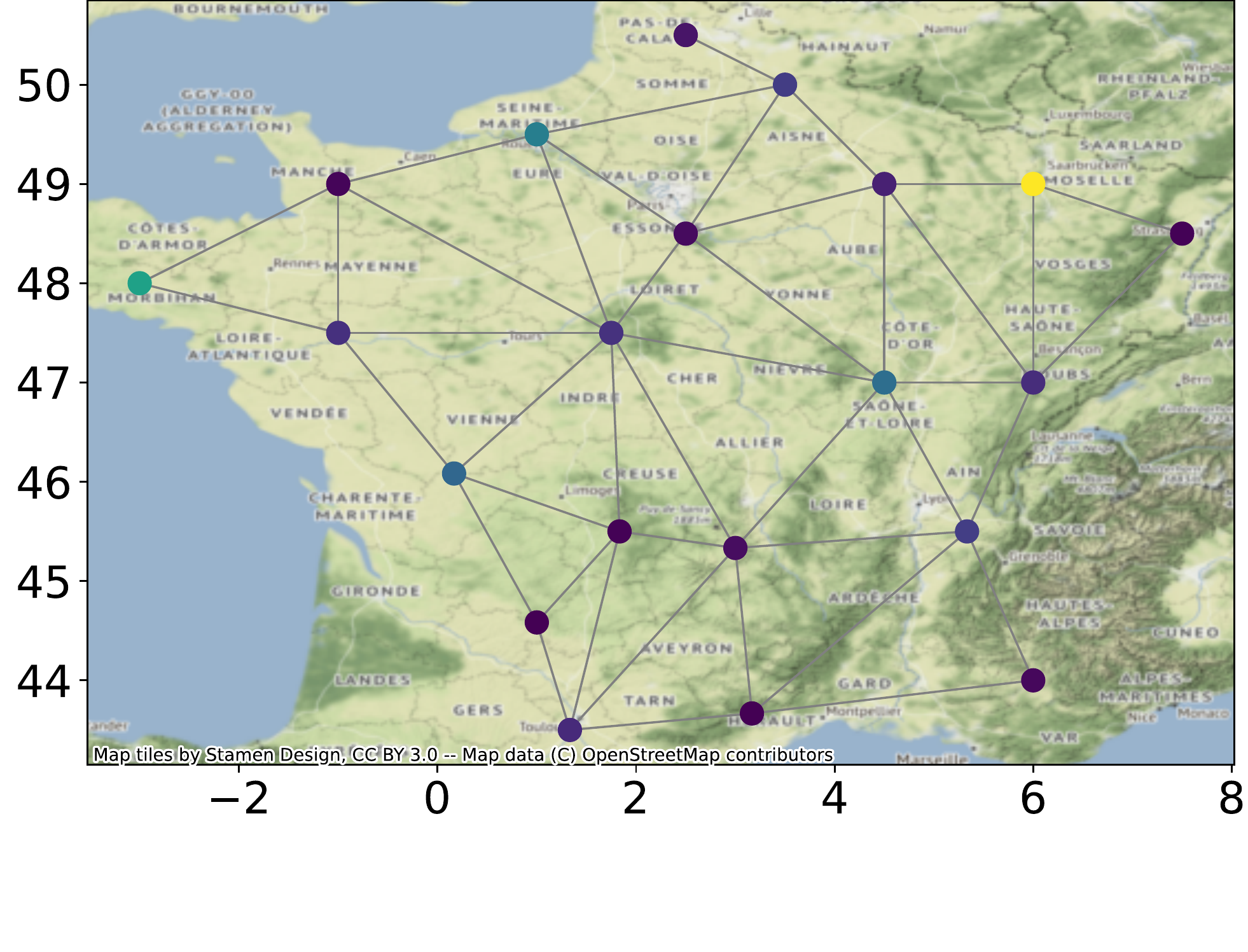}%
	\hspace{2mm}
	\includegraphics[width=0.437\linewidth, viewport=600 10 1135 470,clip]{images/flu_map.pdf}%
	\hspace{1mm}
	\includegraphics[width=0.053\linewidth, viewport=1300 -153 1425 550,clip]{images/flu_map.pdf}
	}
  \includegraphics[width=1.2\linewidth]{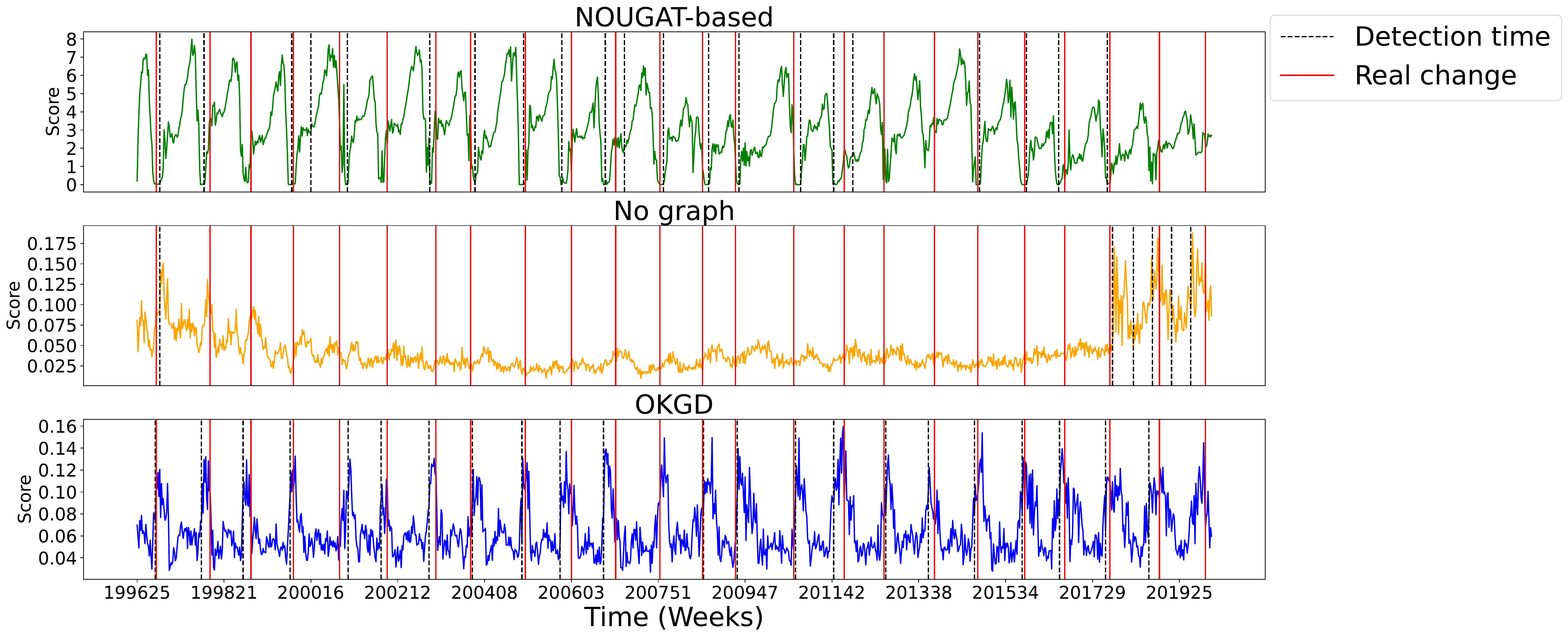}
   \includegraphics[width=1.2\linewidth]{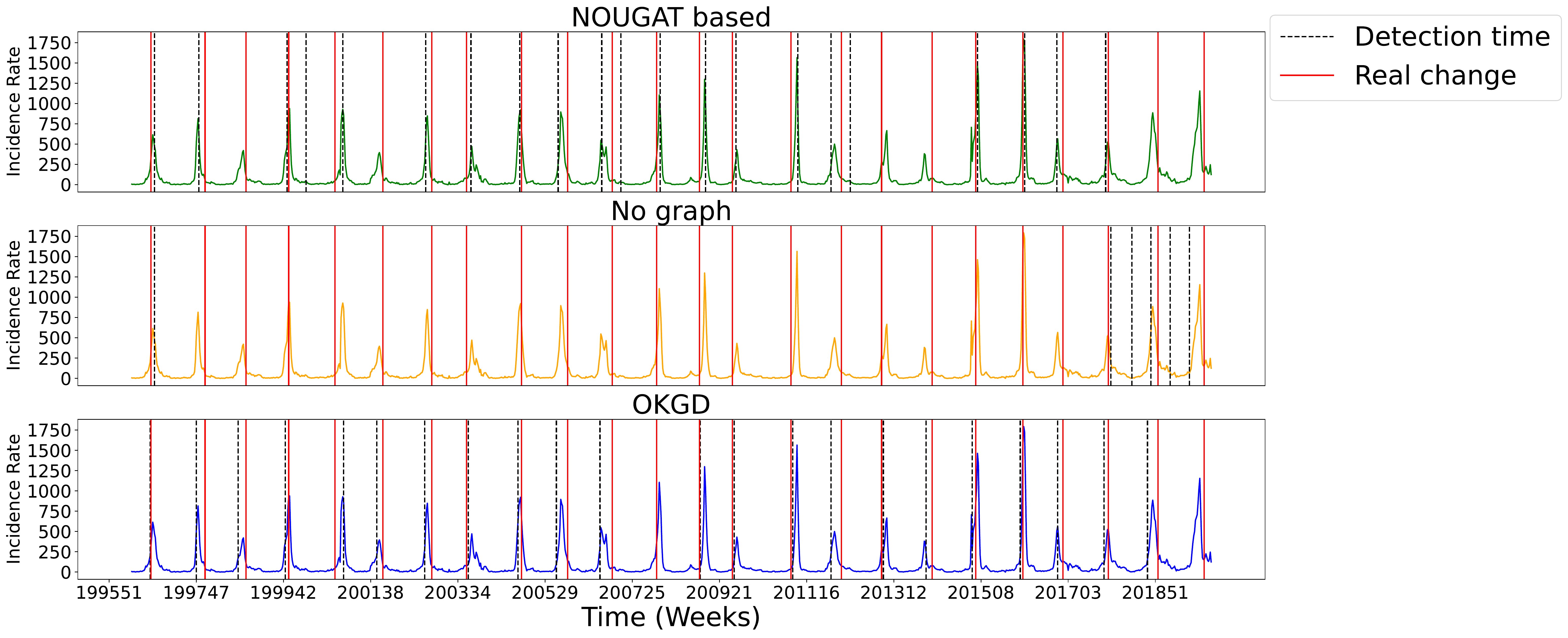}
\end{minipage}
\vspace{-6mm}
\caption{(\textbf{Top}) Comparison of the vector $\hat{g}^2$, $5$ weeks before an outbreak and at a change-point detected by \OKGD. The node color indicates the value of $\hat{g}^2$ at each region of France. (\textbf{Center}) Norm of the score generated by the different change-point detection methods at each time $t$. (\textbf{Bottom}) The national incidence rate per $100$k inhabitants, as reported by the Sentinels Network.}
\label{fig:flu}
\end{figure}
\begin{table}[h!]
\footnotesize
\centering
\makebox[\linewidth][c]{%
\scalebox{.883}{
\begin{tabular}{c r || r r r }
    \toprule
 \textbf{\ Dataset\,/} &  & \textbf{Detection}  
&  \textbf{\#\,False}  
&   
\\
\textbf{\ Scenario} & \textbf{\ Detector} & \textbf{\ delay (std)}  
&  \textbf{alarms}  
&  \textbf{Precision}  
\\
    \hline\hline
\emph{Synthetic}:  &   \ NOUGAT-based & 59.30 (25.70) &   3 \ \ & \ \  \ \ 36 \% \\
Random &     \ \OKGD-no graph & 86.78 (07.28)  &  0  \ \ & \ \ \ \ 60 \%  \\
locations &      \ \OKGD  & 54.79 (13.61) &  0  \ \ & \ \ \ \ 98 \%  \\
 \hline
  \emph{Synthetic}:&     \ NOUGAT-based & 49.56 (25.33) & 7  \ \ &  \ \  \ \ 57 \%   \\
One cluster &     \ \OKGD-no graph &  58.29 (06.47) & 0  \ \ &  \ \ \ \ 48 \% \\
&      \ \OKGD  & 34.29 (14.74)  &  0 \ \ &  \ \ \ \ 75 \%  \\
\hline  
  \emph{Real}: & \ NOUGAT-based & 12.39 seconds  & 3 \ \ & \ \ 100 \% \\
Seismic &     \ \OKGD-no graph & 30.57 seconds  &  0 \ \ & \ \ 100 \% \\ 
&   \ \OKGD  & 12.07 seconds &   0 \ \ & \ \ 100 \% \\  
\hline 
  \emph{Real}: & \ NOUGAT-based & \ \ 2.5 (03.00) weeks &   3 \ \ & \ \  
	75 \% \\
Influenza &     \ \OKGD-no graph &   \ \ 2.33 (01.70) weeks &  3  \ \
&  \ \ 
13 \%  \\  &  \ \OKGD  & \ \ 0.79 (02.06) weeks & 0 \ \ & \ \ 100 \%   \\ 
\hline 
    \bottomrule
\end{tabular}
}
}
\vspace{-2mm}
\caption{Performance comparison between change-point detectors in different scenarios. Whenever there is more than one change-point, the standard deviation of the detection delay is shown in parentheses. The measure of time is specified when relevant.}\label{tab:synthetic experiments}
\end{table}

\newpage
\section{Acknowledgment}
This work was funded by the IdAML Chair hosted at ENS Paris-Saclay, University Paris-Saclay and the DIM Math Innov network.

{\balance
\bibliography{paper}
}

\ \newpage
\appendix

\section{Appendix}

\subsection{Derivation of the cost function of quadratic form}

Here we detail how \Eq{eq:final_problem} is derived in Sec.\,``Cost function''. 
Using the expressions and terms provided in that section, we can reformulate our estimation \Problem{eq:empirical_problem} to:
\begin{equation}\label{eq:final_problem_appendix1}
\begin{aligned}
  &  \min_{\theta \in \R^L } F_t(\theta)  = \min_{\theta \in \R^{L} } \frac{\theta^\top H_t^{\prechange} \theta}{2} +\theta^\top h_t^{\prechange} - \theta^\top h_t^{\postchange} \\ & + \frac{\lambda}{2} \theta^\top \Big( \sum_{y_j \in\setpre} \frac{ K_G(y_j) \Lpc K_G(y_j)^\top}{\Npre} \Big) \theta + \gamma \frac{\theta^\top \id_L \theta}{2} \\
  &  = \min_{\theta \in \Theta } \sum_{v \in V} \left[ \frac{ \theta_v^\top H_{v,t}^{\prechange}\theta_v }{2}  + \theta_v^\top h_{v,t}^{\prechange} -  \theta_v^\top h_{v,t}^{\postchange}  + \frac{\gamma}{2}  \theta_v^\top \id_{L_v} \theta_v
\right] \\ & +  \frac{\lambda}{4} \sum_{y_j \in\setpre }\sum_{u,v \in V} \frac{   W_{uv}  \left(  \theta_v^\top k_v(y_{v,j}) - \theta_u^\top k_u(y_{u,j})  \right)  ^2 }{\Npre}.
\end{aligned}
\end{equation}
And now it is easy to see that $F_t(\theta)$ has the quadratic form:
\begin{equation}\label{eq:final_problem_appendix2}
    \min_{\theta \in \R^L } F_t(\theta)  = \min_{\theta \in \R^{L} } \frac{ \theta^\top  A_t \theta}{2} + \theta^\top b_t,
\end{equation}
%
%
\begin{align}
    \!\!\!\!\text{where}\ \  A_t &= H_t^{\prechange}+  \lambda  \sum_{y_j \in\setpre } \frac{ K_G(y_j) \Lpc K_G(y_j)^\top}{\Npre} + \gamma \id, \\
    b_t &= h_{t}^{\prechange} - h_{t}^{\postchange}.
\end{align}

\subsection{Proof of Theorem \ref{Th:main_result}}

\begin{proof}
Notice we can write $\ExpecNA{A_t}= C_t+ \lambda \id $, where $C_t$ is a positive definite positive matrix thanks to the kernel and Laplacian positive definiteness. This implies $\ExpecNA{A_t}$ is positive definite, then invertible. This means \Problem{eq:expected_cost} admits a unique solution given by: 
\begin{equation}
    \theta^*=\ExpecNA{A_t}^{-1} \ExpecNA{b_t}.
\end{equation}
Notice: when there is no change-point $\ExpecNA{b_t}^{-1}=\ExpecN{h^{\prechange}_{v,t}}-\ExpecN{h^{\postchange}_{v,t}}=0$, and hence $\theta^*=0$
\end{proof}

For completeness of the proof we include the results appearing in the BSGD's paper  \cite{Xu2015}. In this work the objective function to optimize takes the form:
\begin{equation}\label{eq:XU_opt_problem}
\begin{aligned}
    \min_{x \in \chi_v} \Phi(x) &=  \min_{x \in \chi_v}  \Expec{f(x,\epsilon)} + \sum_{v=1}^{N} r(x_v) \\
     &=  \min_{x \in \chi_v}  \textbf{F}(x)+ \sum_{v=1}^{N} r(x_v) \\
     &  \ \ \  \text{s.t}   \ \ \  x_v \in \chi_v  \ \ \ v=1,...,N,
\end{aligned}
\end{equation}
where $\chi_v \subset \R^{n_v}$, the variable $x$ is of dimension $L=\sum_{v=1}^s L_v$ and is partitioned into disjoint blocks $x=(x_1,x_2,...,x_N)$, $\epsilon$ is a random variable. 

The convergence of the $t$-th update described in \Expr{eq:update} is guaranteed under the following hypothesis:

\begin{assumption_2}\label{ass_a:boundness}
There exist a constant $A$ and a sequence $\{\sigma_k\}$ such that, for any $v$ and $t$,
\begin{equation}
\begin{aligned}
\norm{\Expec{\delta_{v,t}|\Xi_{t-1}} }\leq A \max_{v} \alpha_{v,t} \\
\Expec{\norm{\delta_{v,t}}}^2 \leq \sigma_t^2,
\end{aligned}  
\end{equation}
where $\Xi_{t-1}$ refers to the set of variables sampled at the interaction $t-1$.
\end{assumption_2}
%

\begin{assumption_2}\label{ass_a:Lipschitz}
The objective function is lower bounded, \ie $\Phi(x)>-\infty$ . There exists a uniform Lipschitz constant $L>0$ such that:
\begin{equation}
\norm{\Delta_{x_v}F(x) - \Delta_{x_v}F(y)} \leq L \norm{x-y}, \ \ \  \forall x,y, \forall i.
\end{equation}
\end{assumption_2}

\begin{assumption_2}\label{ass_a:boundness_x}
There exists a constant $\rho$ such that $\Expec{x_t}^2 \leq \rho^2 \forall t$.
\end{assumption_2}

\begin{assumption_2}\label{ass_a:Lipschitz_r}
Every function $r_i$ is Lipschitz continuous, namely there is a constant $L_{r_v}$, such that $\norm{r_v(x_v)-r_v(y_v)} \leq C_{r_v} \norm{x_v-y_v}, \forall x_v,y_v$.
\end{assumption_2}
They denote by $L_{\text{max}}=\max_v C_{r_v}$ the dominant Lipschitz constant.

One of the main results of \cite{Xu2015} is the convergence of 
BSGD for a set of cost functions $\phi(x)$:

\begin{theorem_2}\label{Th:xu2015}
 \cite{Xu2015} Let $\{x_t\}$ be generated from the BSGD algorithm with $\alpha_{v,t}=\alpha_{t}=\frac{c}{t}$, $\forall v,t$. Under Assumptions, if $F$ and $r_v$ are all convex, $\Psi$ is strongly convex with modulus $\mu>0$, and $\sigma= \sup_t{\sigma_t} < \infty$, then: 
\begin{equation}
\!\!\!\!\Expec{\norm{x_t-x^*}^2} \leq \frac{1}{t} \max\left\{\frac{2 D c(1+\mu c)}{\mu},\norm{x_1-x^*}\right\},\!\!
\end{equation}
where $x^*$ is solution to the optimization \Problem{eq:XU_opt_problem} and $D$ is defined as: 
\begin{equation}
\begin{aligned}
   \!\! D &= \frac{N (\sigma^2+4L_{\max})}{1-Lc} \\
   & +\sigma(\norm{x^*}+\rho)\left(\! A+ L \sqrt{\sum_{v=1}^N(4M_{\rho}^2+4\sigma^2_t+2L_{r_v}^2)} \right)\!\!\!\!\!
\end{aligned}
\end{equation} 
%
%
\begin{equation}
   \!\text{and}\ \ \  M_{\rho}=\sqrt{4L^2\rho^2 +2 \max_{v} \norm{\Delta_{x_v}F(0)}^2 }.
\end{equation} 
 
\end{theorem_2}

\subsection{Proof of Theorem
\ref{Th:convergence}}

In this section, we will prove the assumptions described are satisfied under the hypothesis $p(y) \neq p'(y)$. This will guarantee the convergence of our method as stated in  \Theorem{Th:xu2015}. To improve the clarity of the demonstration, we define the terms 
%
$\Hilbert_v=\ExpecN{H_{v,t}}$ and $\Hilbert_{vu}=\ExpecN{k_v(y_{v,t})k_u(y_{u,t})^\top}$, 
where $\lambda^{\text{max}}_v$ is the maximum eigenvalue of $ \Hilbert_v$.

On the other hand, when $A$ is a matrix, we will denote by $\norm{A}_2$ the spectral norm of $A$. 

According to the proof of \Theorem{Th:main_result},
the optimization, when there is not change-point, \Problem{eq:expected_cost} has a unique solution  $\theta^*=0$ and $\mathbb{E}_{p(y)}[b_t]=0$. This means that we can reduce the optimization domain to a ball of radius $\rho>0$ ($\mathbb{B}_{\rho}$) and rewrite \Problem{eq:expected_cost} as:
\begin{equation}
\begin{aligned}
\min_{\theta \in \R^{L}} \textbf{F}(\theta) &= \min_{\theta \in \mathbb{B}_{\rho} }  \theta^\top  \mathbb{E}_{p(y)}  \Bigg[ \frac{H_t^{\prechange}}{2} +  \frac{\gamma}{2} \id_L \\
& +\frac{\lambda}{2}  \Big(\!\!\sum_{j \in Y^{\prechange}}  \frac{K_G(y_j) \Lpc K_G(y_j)^\top}{\Npre}\Big)  \Bigg]  \theta. \\
\end{aligned}
\end{equation}
This cost function is always strongly convex with modulus  $\frac{\gamma}{2}$, and it is lower-bounded by $0$. 

Additionally, it is easy to see that the function $\dtheta \mathbf{F}_t(x)$ is a linear function in $x$, then Lipschitz continuous. Furthermore we can show that: 
\begin{equation}
\begin{aligned}
& \norm{\dtheta \mathbf{F}(x)-\dtheta \mathbf{F}(y)} \leq \norm{((1+\lambda d_v)\Hilbert_v+\gamma I)(x_v-y_v)} \\ 
& +  \lambda \sum_{u \in \nghood(v)}
W_{uv} \norm{\Hilbert_{vu}(x_u-y_u)}
\\
& \leq \left[ (1+\lambda d_v)  \lambda^{\text{max}}_v+\gamma \right]
\norm{x_v-y_v} \\
& + \lambda M_v \sum_{u \in \nghood(v)} W_{uv}  M_u \norm{x_u-y_u}  \\
& \leq \left[ (1+\lambda d_v)  \lambda^{\text{max}}_v+\gamma
+ \lambda M_v \sum_{u \in \nghood(v)} W_{uv}  M_u \right] \norm{x-y} \\
& = C_v \norm{x-y} \\
&  \leq L \norm{x-y}.
\end{aligned}
\end{equation}
%
The second inequality is true thanks to \Assumption{assum:kernel} and 
\begin{equation}{\label{eq:C_v}}
\begin{aligned}
C_v &=   (1+\lambda d_v)  \lambda^{\text{max}}_v+\gamma
+ \lambda M_v \sum_{u \in \nghood(v)} W_{uv}  M_u  \\
L &= \max_{v \in V } C_v.
\end{aligned}
\end{equation}
The previous discussion proves that the conditions of Assumption A.\ref{ass_a:Lipschitz} are satisfied.

Regarding Assumption A.\ref{ass_a:boundness_x}, the equivalent formulation of the optimization problem in the set $\mathbb{B}_{\rho}$ allow us replace our gradient-step for a projected gradient-step in $\mathbb{B}_{\rho}$. In other words, we can always guarantee that: 
\begin{equation}\label{eq:theta_bound}
    \norm{\theta_t} < \rho,
\end{equation}
which is a more restrictive condition that Assumption A.\ref{ass_a:boundness_x}.

Now lets validate that Assumption 3 is satisfied. Notice that the terms $h_{v,t}^{\postchange},h_{v,t}^{\prechange},H_{v,t}^{\prechange}$ are independent from $\setprebef$ given the independence assumption in the time-series and the fact that $Y_{t-1}^{\prechange}$ is sampled uniformly at random $Y^{\prechange}$ at every time stamp. This reduces the complexity of some terms: 
\begin{equation}
\begin{aligned}
&\ExpecN{\Delta_{\theta_{v}} F_t((\theta_{t+1})_{<v},(\theta_{t})_{ \geq v})|\setprebef} = \\
& \ExpecN{ h_{v,t}^{\postchange}-h_{v,t}^{\prechange}|\setprebef} + (1+\lambda d_v) \ExpecN{H_{v,t}^{\prechange} \theta_{v,t}|\setprebef}
\\ &+\gamma \ExpecN{\theta_{v,t}|\setprebef}  - \lambda \sum_{u \in \nghood(v)} W_{vu} \frac{1}{\Npre} * 
\\
& \sum_{y_j \in\setpre}
\Big[ \ExpecN{k_v(y_{v,j})k_u(y_{u,j})^\top \theta_{u,t+1} \one_{u<v} |\setprebef} 
\\& + \ExpecN{k_v(y_{v,j})k_u(y_{u,j})^\top \theta_{u,t} \one_{u \geq v} |\setprebef} \Big] 
\\
  &=(1+\lambda d_v) \ExpecN{H_{v,t}^{\prechange}} \ExpecN{ \theta_{v,t}|\setprebef} +\gamma \ExpecN{\theta_{v,t}|\setprebef} \\ & - \lambda \sum_{u \in \nghood(v)} W_{vu} \Big[
   \ExpecN{k_v(y_{v,t})k_u(y_{u,t})^\top} \ExpecN{\theta_{u,t} \one_{u \geq v} |\setprebef} \\
  & + \frac{1}{\Npre}  \sum_{y_j \in\setpre} \ExpecN{k_v(y_{v,j})k_u(y_{u,j})^\top \theta_{u,t+1} \one_{u < v} |\setprebef}\Big]
 \\
  &=(1+\lambda d_v) \Hilbert_v \ExpecN{ \theta_{v,t}|\setprebef} +\gamma \ExpecN{\theta_{v,t}|\setprebef}  \\ & - \lambda \sum_{u \in \nghood(v)} W_{vu} \Big[
  \Hilbert_{vu} \ExpecN{\theta_{u,t}   \one_{u \geq v} |\setprebef} \\
  & + \frac{1}{\Npre}  \sum_{y_j \in\setpre} \ExpecN{k_v(y_{v,j})k_u(y_{u,j})^\top \theta_{u,t+1} \one_{u < v} |\setprebef}\Big],
\end{aligned}
\end{equation}
where the second equality follows from the independence hypothesis. On the other hand: 
\begin{equation}
\begin{aligned}
& \ExpecN{\Delta_{\theta_{v}}\textbf{F}((\theta_{t+1})_{<v},(\theta_{t-1})_{ \geq v})|\setprebef}=  \\
&  (1+\lambda d_v) \Hilbert_v \ExpecN{ \theta_{v,t}|\setprebef} +\gamma \ExpecN{\theta_{v,t}|\setprebef}  \\ & - \lambda \sum_{u \in \nghood(v)} W_{vu} \Hilbert_{vu} \Big[
 \ExpecN{\theta_{u,t}   \one_{u \geq v} |\setprebef} 
 \\ &   +  \ExpecN{ \theta_{u,t+1} \one_{u < v} |\setprebef}\Big].
\end{aligned}
\end{equation}
Then, it is easy to obtain the upper-bounds for the quantity:
$$\delta_{v,t}=\Delta_{\theta_{v}} F_t((\theta_{t+1})_{<v},(\theta_{t})_{ \geq v})-\Delta_{\theta_{v}}\textbf{F}((\theta_{t+1})_{<v},(\theta_{t-1})_{ \geq v})$$
\begin{equation}
\begin{aligned}
& \norm{\ExpecN{\delta_{v,t}|\setprebef}}=  \lambda \Bigg\lVert \sum_{u \in \nghood(v)} W_{vu}  \frac{1}{\Npre}* \\
&\sum_{y_j \in\setpre} \Big[ \ExpecN{k_v(y_{v,j})k_u(y_{u,j})^\top \theta_{u,t+1} \one_{u < v} |\setprebef}
\\ & - \Hilbert_{vu} 
 \ExpecN{ \theta_{u,t+1} \one_{u < v} |\setprebef}\Big] \Bigg\rVert \\
&= \lambda   \Bigg\lVert  \sum_{u \in \nghood(v)} W_{vu}  \frac{1}{\Npre}* \\
&\sum_{y_j \in\setpre} \ExpecN{\big(k_v(y_{v,j})k_u(y_{u,j})^\top-\Hilbert_{vu}\big) \theta_{u,t} \one_{u < v} |\setprebef} \\ 
&- \alpha_{u,t}  \ExpecN{\big(k_v(y_{v,j})k_u(y_{u,j})^\top-\Hilbert_{vu}\big)* \\
& \Delta_{\theta_{v}} F_t((\theta_{t+1})_{<u},(\theta_{t})_{ \geq u})) | \setprebef}  \Bigg\rVert \ \  \explanation{Update rule for $(\theta_{t+1})_{<u}$}
\\ &=   \lambda  \Bigg\lVert  \sum_{u \in \nghood(v)} W_{vu} \alpha_{u,t}  \frac{1}{\Npre}*   \\
&  \sum_{y_j \in\setpre} \ExpecN{\big(k_v(y_{v,j})k_u(y_{u,j})^\top-\Hilbert_{vu}\big)* \\ & \Delta_{\theta_{v}} F_t((\theta_{t+1})_{<u},(\theta_{t})_{ \geq u}) |\setprebef} \Bigg\lVert  \ \ \ \explanation{\Assumption{assum:independence}} \\
&=  2 \lambda  M_v \sum_{u \in \nghood(v)} W_{vu} \alpha_{u,t} M_u * \\  & \left\lVert \ExpecN{\Delta_{\theta_{u}} F_t((\theta_{t+1})_{<u},(\theta_{t})_{ \geq u}) |\setprebef)} \right\rVert \explanation{\Assumption{assum:kernel}} \\
 &  \leq 2 M_v \lambda \sum_{u \in \nghood(v)} W_{uv}  M_u \alpha_{u,t} \Big[ (1+\lambda d_v) \norm{
  \Hilbert_v \ExpecN{ \theta_{v,t}|\setprebef}} \\
 &  +\gamma \norm{\ExpecN{\theta_{v,t}|\setprebef}} \\
 & + \lambda \sum_{u \in \nghood(v)} W_{vu} \Bigg[
\norm{  \Hilbert_{vu} \ExpecN{\theta_{u,t}   \one_{u \geq v} |\setprebef} } \\
  & + \frac{1}{\Npre} \Big\Vert\!\!\sum_{y_j \in\setpre} \ExpecN{k_v(y_{v,j})k_u(y_{u,j})^\top \theta_{u,t+1} \one_{u < v} |\setprebef}\Big\Vert\Bigg]
    \\ &  \leq 2 \lambda M_v  \sum_{u \in \nghood(v)} W_{uv}  M_u \alpha_{u,t} C_u \rho  \ \ \ \explanation{\Expr{eq:theta_bound} and \Expr{eq:C_v}}
     \\ &  \leq  2 \lambda   \rho  M_v  \Big(\! \sum_{u \in \nghood(v)} W_{uv}  M_u C_u \!\!\Big)\max_{j} \alpha_{j,t}
     \\ & \leq  A  \max_{j} \alpha_{j,t},
\end{aligned}
\end{equation}
where $A=\max_{v \in V} 2 \lambda   \rho  M_v  \left( \sum_{u \in \nghood(v)} W_{uv}  M_u C_u \right)$.

On the other side we have that: 
\begin{equation}
\begin{aligned}
&  \ExpecN{ \norm{\delta_{v,t}}^2}  
\leq \mathbb{E}_{p(y)} \Bigg[  \Big\lVert (1+\lambda d_v) \left( H_{v,t}-\Hilbert_v\right)\theta_{v,t} \\
& + \lambda \sum_{u \in \nghood(v)} W_{uv} \frac{1}{\Npre} \sum_{y_j \in\setpre} \big(k_v(y_{v,j})k_u(y_{u,j})^\top-\Hilbert_{vu} \big)* \\
& \big(\theta_{u,t}   \one_{u \geq v}+\theta_{u,t+1} \one_{u < v} \big) +  h_{v,t}^{\postchange}-h_{v,t}^{\prechange} \Big\rVert^2
\Bigg] \\
 & \leq 4(1+\lambda d_v)^2  \ExpecN{\norm{H_{v,t}-\Hilbert_v}^2 \norm{\theta_{v,t}}^2} \\ 
& + 4 \lambda^2 \mathbb{E}_{p(y)} \Bigg[\Big\lVert \sum_{u \in \nghood(v)} W_{uv} \frac{1}{\Npre} \sum_{y_j \in\setpre} \big(k_v(y_{v,j})k_u(y_{u,j})^\top-\Hilbert_{vu} \big)* \\
 & \big(\theta_{u,t}   \one_{u \geq v}+\theta_{u,t+1} \one_{u < v} \big)\Big\rVert^2 \Bigg] + 2 \ExpecN{\norm{h_{v,t}^{\postchange}-h_{v,t}^{\prechange}}^2} \\
& \explanation{By applying twice the expression $ab \leq a^2 +b^2$} \\
&  \leq 4(1+\lambda d_v)^2 \rho^2 \ExpecN{\norm{H_{v,t}-\Hilbert_v}^2_2} + 2 \ExpecN{\norm{ h_{v,t}^{\postchange}-h_{v,t}^{\prechange}}^2} 
 \\ & + 4 \lambda^2 \rho^2 \mathbb{E}_{p(y)} \Bigg[\Bigg( \sum_{u \in \nghood(v)} W_{uv} \Bigg\lVert \frac{1}{\Npre}* \\
& \sum_{y_j \in\setpre} \big(k_v(y_{v,j})k_u(y_{u,j})^\top-\Hilbert_{vu}\big) \Bigg\rVert_2 \Bigg)^2 \Bigg] \ \ \ \explanation{\Expr{eq:theta_bound}} \\
& = \sigma_{v,t}
\end{aligned}
\end{equation}
Given Assumption \ref{assum:kernel}, we can define $\sigma_{t}=\max_{v \in V} \sigma_{v,t} $ and $\sigma_{t}<\infty$. Then the upper-bounds appearing in \Assumption{ass_a:boundness} are satisfied. 

Lastly, we can see how in our case the function $r_i(x)=0$ and $f(x,\epsilon)=F_t(x)$. This implies that: 
\begin{equation}
   M_{\rho}=2L\rho, 
\end{equation} 
which leads to: 
\begin{equation}
  D=\frac{N \sigma^2}{1- L c} + \sqrt{N \rho^2}
\Big( A+ L \sqrt{ N(4 L^2 \rho^2  +4\sigma^2_t)} \Big).
\end{equation} 

As we check the proof of Theorem A.\ref{Th:xu2015} we arrive to an intermediate step that appears as Eq.\,30 in \cite{Xu2015}, and leads to the upperbound: 
\begin{equation*}
    \ExpecN{\norm{\theta_{t+1}}^2} \leq  \frac{1}{t} \max\{\frac{4D c(1+0.5 c \gamma)}{\gamma} ),\norm{\theta_{1}}^2\}.
\end{equation*}

Finally, the last inequality shown in the formulation of \Theorem{Th:convergence} is a rewriting of the inequality appearing in Theorem A.\ref{Th:xu2015} in our context.

\newpage

\subsection{Proof of Corollary \ref{cor:convergence}}

We can deduce that under the null hypothesis: 
\begin{equation}
\begin{aligned}
\ExpecN{\norm{\hat{g}_t}^2} & \leq \ExpecN{\theta_t K_G(y_j)K_G(y_j)^\top \theta_t^\top} \\
& \leq \left(\!\sum_{v \in V } L_v M_v^2 \!\right) \ExpecN{\norm{\theta_t}^2} \ \ \ \explanation{\Assumption{assum:kernel}}\\
 & \leq \frac{1}{t} \left(\! \sum_{v \in V } L_v M_v^2 \!\right)  \max\left\{\!\frac{4D c(1+0.5 c \gamma)}{\gamma},\norm{\theta_{1}}^2\!\right\}. 
\end{aligned}
\end{equation}

The last expression is a consequence of \Theorem{Th:convergence}, the we can conclude that $\ExpecN{\norm{\hat{g}_t}^2} \rightarrow 0$ as $t \rightarrow \infty$, meaning our estimator is an asymptotically unbiased under the null hypothesis.

\subsection{Details for the synthetic data}

The synthetic data we used in our experiments are by construction heterogeneous, since the time-series of the nodes in each cluster are generated by different probabilistic models. More specifically: 
\begin{itemize}
    \item C$1$: 
		bivariate a Gaussian distribution with mean vector equal to the vector of zeros and with covariance matrix equal to the identity matrix. 
		\item C$2$: a Poisson distribution with expected value fixed at $5$.
		\item C$3$:
		a bivariate Gaussian distribution with mean vector equal to the vector of zeros, variance values fixed at $1$, and the covariance 
		fixed at $0.75$.
		\item C$4$: a Poisson distribution with expected value fixed at $10$.
\end{itemize}

\end{document}